\documentclass[11pt,english]{article}
\usepackage{lmodern}

\usepackage[T1]{fontenc}
\usepackage[latin9]{inputenc}
\usepackage{geometry}
\usepackage{xcolor}
\usepackage{pdfcolmk}
\usepackage{babel}
\usepackage{amsmath}
\usepackage{amsthm}
\usepackage{amssymb}
\usepackage{graphicx}
\usepackage{esint}
\PassOptionsToPackage{normalem}{ulem}
\usepackage{ulem}
\usepackage[unicode=true,pdfusetitle,
 bookmarks=true,bookmarksnumbered=false,bookmarksopen=false,
 breaklinks=false,pdfborder={0 0 1},backref=false,colorlinks=false]
           {hyperref}
\usepackage{breakurl}

\makeatletter

\providecolor{lyxadded}{rgb}{0,0,1}
\providecolor{lyxdeleted}{rgb}{1,0,0}
\DeclareRobustCommand{\lyxdeleted}[3]{{\texorpdfstring{\color{lyxdeleted}\sout{#3}}{}}}

\numberwithin{equation}{section}
\numberwithin{figure}{section}
\numberwithin{table}{section}
\newcommand{\lyxaddress}[1]{
\par {\raggedright #1
\vspace{1.4em}
\noindent\par}
}
\theoremstyle{plain}
\newtheorem{thm}{\protect\theoremname}
  \theoremstyle{plain}
  \newtheorem{lem}[thm]{\protect\lemmaname}
  \theoremstyle{remark}
  \newtheorem*{rem*}{\protect\remarkname}

\newcommand {\norm} [1] { \lVert #1 \rVert}

\newcommand {\abs} [1] {\left| #1 \right|}
\newcommand {\Abs} [1] {\bigl\lvert #1 \bigr\rvert}

\makeatother

  \providecommand{\lemmaname}{Lemma}
  \providecommand{\remarkname}{Remark}
\providecommand{\theoremname}{Theorem}

\begin{document}

\title{Prediction of dynamical time series using kernel based regression
and smooth splines}

\author{Raymundo Navarrete and Divakar Viswanath }

\maketitle

\lyxaddress{Department of Mathematics, University of Michigan (raymundo@umich.edu)
(divakar@umich.edu). }

\date{$\:$}
\begin{abstract}
Prediction of dynamical time series with additive noise using support
vector machines or kernel based regression is consistent for certain
classes of discrete dynamical systems. Consistency implies that these
methods are effective at computing the expected value of a point at
a future time given the present coordinates. However, the present
coordinates themselves are noisy, and therefore, these methods are
not necessarily effective at removing noise. In this article, we consider
denoising and prediction as separate problems for flows, as opposed
to discrete time dynamical systems, and show that the use of smooth
splines is more effective at removing noise. Combination of smooth
splines and kernel based regression yields predictors that are more
accurate on benchmarks typically by a factor of 2 or more. We prove
that kernel based regression in combination with smooth splines converges
to the exact predictor for time series extracted from any compact
invariant set of any sufficiently smooth flow. As a consequence of
convergence, one can find examples where the combination of kernel
based regression with smooth splines is superior by even a factor
of $100$. The predictors that we analyze and compute operate on delay
coordinate data and not the full state vector, which is typically
not observable.
\end{abstract}

\section{Introduction}

\global\long\def\norm#1{\left|\left|#1\right|\right|}

\global\long\def\inner#1#2{\left<#1,#2\right>}

\global\long\def\abs#1{\left|#1\right|}

\global\long\def\Abs#1{\Bigl|#1\Bigr|}

\global\long\def\argmin{\mathrm{argmin\,}}

\global\long\def\halfi#1{\left[#1\right)}

The problem of time series prediction is to use knowledge of a signal
$x(t)$ for $0\le t\leq T$ and infer its value at a future time $t=T+t_{f}$,
where $t_{f}$ is positive and fixed. A time series is not predictable
if it is entirely white noise. Any prediction scheme has to make some
assumption about how the time series is generated. A common assumption
is that the observation $x(t)$ is a projection of the state of a
dynamical system with noise superposed \cite{FarmerSidorowich1987}.
Since the state of the dynamical system can be of dimension much higher
than $1$, delay coordinates are used to reconstruct the state. Thus,
the state at time $t$ may be captured as 
\begin{equation}
\left(x(t),x(t-\tau),\ldots,x(t-(D-1)\tau)\right)\label{eq:s1-delay-coord}
\end{equation}
where $\tau$ is the delay parameter and $D$ is the embedding dimension.
Delay coordinates are (generically) effective in capturing the state
correctly provided $D\geq2d+1$, where $d$ is the dimension of the
underlying dynamics \cite{SauerYorkeCasdagli1991}.

Farmer and Sidorowich \cite{FarmerSidorowich1987} used a linear framework
to compute predictors applicable to delay coordinates. It was soon
realized that the nonlinear and more general framework of support
vector machines would yield better predictors \cite{MatteraHaykin1999,MukherjeeOsunaGirosi1997,MullerSmolaRatschScholkopfKohlmorgenVapnik1998}.
Detailed computations demonstrating the advantages of kernel based
predictors were given by Müller et al \cite{MullerSmolaRatschScholkopfKohlmorgenVapnik1998}
and are also discussed in the textbook of Schölkopf and Smola \cite{ScholkofSmola2002}.
Kernel methods still appear to be the best, or among the best, for
the prediction of stationary time series \cite{McGoffMukherjeePillai2015,SapankevychSankar2009}.

A central question in the study of noisy dynamical time series is
how well that noise can be removed to recover the underlying dynamics.
Lalley, and later Nobel, \cite{Lalley1999,Lalley2001,LalleyNobel2006}
have examined hyperbolic maps of the form $x_{n+1}=F(x_{n})$, with
$F:\mathbb{R}^{d}\rightarrow\mathbb{R}^{d}$. It is assumed that observations
are of the form $y_{n}=x_{n}+\epsilon_{n}$, where $\epsilon_{n}$
is iid noise. They proved that it is impossible to recover $x_{n}$
from $y_{n}$, even if the available data $y_{n}$ is for $n=0,-1,-2,\ldots$
and infinitely long, if the noise is normally distributed. However,
if the noise satisfies $\abs{\epsilon_{n}}<\Delta$ for a suitably
small $\Delta$, the underlying signal $x_{n}$ can be recovered.
The recovery algorithm does not assume any knowledge of $F$. The
phenomenon of unrecoverability is related to homoclinic points. If
the noise does not have compact support, with some nonzero probability,
it is impossible to distinguish between homoclinic points. 

Lalley \cite{Lalley2001} suggested that the case of flows could be
different from the case of maps. In discrete dynamical systems, there
is no notion of smoothness across iteration. In the case of flows,
the underlying signal will depend smoothly on time but the noise,
which is assumed to be iid at different points in time, will not.
Lalley's algorithm for denoising relies on dynamics and, in particular,
on recurrences. In the case of flows, we rely solely on smoothness
of the underlying signal for denoising. As predicted by Lalley, the
case of flows is different. Denoising based on smoothness of the underlying
signal alone can handle normally distributed noise or other noise
models. Thus, our algorithms are split into two parts: first the use
of smooth splines to denoise, and second the use of kernel based regression
to compute the predictor. Only the second part relies on recurrences.

Prediction of discrete dynamics, within the framework of Lalley \cite{Lalley1999},
has been considered by Steinwart and Anghel \cite{SteinwartAnghel2009}
(also see \cite{ChristmannSteinwart2007}). Suppose $x_{n}=F^{n}(x_{0})$
and $\tilde{x}_{n}=x_{n}+\epsilon_{n}$ is the noisy state vector.
The risk of a function $f$ is defined as 
\[
\int\int\abs{F(x)+\epsilon_{1}-f(x+\epsilon_{2})}^{2}\,d\nu(\epsilon_{1})\,d\nu(\epsilon_{2})\,d\mu(x),
\]
where $\nu$ is the distribution of the noise and $\mu$ is a probability
measure invariant under $F$ and with compact support. Thus, the risk
is a measure of how well the noisy future state vector can be predicted
given the noisy current state vector. It is proved that kernel based
regression is consistent with respect to this notion of risk for a
class of rapidly mixing dynamical systems. Although the notion of
risk does not require denoising, consistency of empirical risk minimization
is proved for additive noise $\epsilon_{n}$ of compact support as
in \cite{Lalley1999}. In the case of empirical risk minimization,
compactness of added noise is not a requirement imposed by the underlying
dynamics but is assumed to make it easier to apply universality theorems.

Our results differ in the following ways. We consider flows and not
discrete time maps. In addition, we work with delay coordinate embedding
\cite{SauerYorkeCasdagli1991} and do not require the entire state
vector to be observable. Finally, we prove convergence to the exact
predictor, which goes beyond consistency. The convergence theorem
we prove is not uniform over any class of dynamical systems. However,
we do not assume any type of decay in correlations or rapid mixing.
Non-uniformity in convergence is an inevitable consequence of proving
a theorem that is applicable to any compact invariant set of a generic
finite dimensional dynamical system \cite{AdamsNobel1998,GyorfiHardleSardaVieu2013,SteinwartHushScovel2009}.
This point is further discussed in section 2, which presents the main
algorithm as well as a statement of the convergence theorem. Section
3 presents a proof of the convergence theorem.

In section 4, we present numerical evidence of the effectiveness of
combining spline smoothing and kernel based regression. The algorithm
of section 2 is compared to computations reported in \cite{MullerSmolaRatschScholkopfKohlmorgenVapnik1998}
and the spline smoothing step is found to improve accuracy of the
predictor considerably. The numerical examples bring up two points
that go beyond either consistency or convergence. First, we explain
heuristically why it is not a good idea to iterate $1$-step predictor
$k$-times to predict the state $k$ steps ahead. Rather, it is a
much better idea to learn the $k$-step predictor directly. Second,
we point out that no currently known predictor splits the distance
vector between stable and unstable directions, a step which was argued
to be essential for an optimal predictor by Viswanath et al \cite{ViswanathLiangSerkh2013}.
The heuristic explanation for why iterating a $1$-step predictor
$k$ times is not a good idea relies on the same principle.

The concluding discussion in section 5 points out connections to related
lines of current research in parameter inference \cite{McGoffMukherjeePillai2015,McGoffNobel2016}
and optimal consistency estimates for stationary data \cite{HangFengSteinwartSuykens2016}.

\section{Prediction algorithm and statement of convergence theorem}

Let $\frac{dU}{dt}=\mathcal{F}(U)$, where $\mathcal{F}\in C^{r}(\mathbb{R}^{d},\mathbb{R}^{d})$,
$r\geq2$, define a flow that may be limited to an open subset of
$\mathbb{R}^{d}$ with compact closure. Let $\mathcal{F}_{t}(U_{0})$
be the time-$t$ map with initial data $U_{0}$. It is assumed that
$U(t;U_{0})$, $t\in\mathbb{R}$, is a trajectory of the flow whose
initial point $U(0;U_{0})$ is $U_{0}\in\mathbb{R}^{d}$. Let $\tilde{\mu}$
be a compactly supported invariant probability measure of the flow-map
$\mathcal{F}_{t}$ for $t>0$ and let $\tilde{X}$ be its support.
It is assumed that the initial point $\tilde{\omega}$ is drawn from
the measure $\tilde{\mu}$. For $\tilde{\omega}\in\tilde{X}$, the
trajectory $U(t;\tilde{\omega})$ exists for all $t\in\mathbb{R}$
and is unique. In addition, the flow is assumed to be ergodic with
respect to the measure $\tilde{\mu}$.

Let $\pi:\mathbb{R}^{d}\rightarrow\mathbb{R}$ be a generic nonlinear
projection. Let $u(t;\tilde{\omega})=\pi U(t;\tilde{\omega})$ be
the projection of the random trajectory $U(t;\tilde{\omega})$. By
the embedding theorem of Sauer et al \cite{SauerYorkeCasdagli1991},
we assume that the delay coordinates give a $C^{r}$ diffeomorphism
into the state space implying that $U(t;\tilde{\omega})$ can be recovered
from the delay vector, with delay $\tau>0$, 
\[
\left(u(t;\tilde{\omega}),u(t-\tau;\tilde{\omega}),\ldots,u(t-(D-1)\tau;\tilde{\omega})\right)
\]
for $D\geq2d+1$. This delay vector is denoted by $u(t;\tau;\tilde{\omega})$. 

As a consequence of the $C^{r}$ embedding, there is a measure $\mu$
compactly supported in $\mathbb{R}^{D}$ that corresponds to $\tilde{\mu}$.
The measure $\mu$ is ergodic and invariant under the flow lifted
via the embedding. Denote the compact support of $\mu$ by $X$. For
every point $\tilde{\omega}$ in $\tilde{X}$, there corresponds a
unique point $\omega$ in $X$ and vice versa. Because the prediction
algorithm is based on delay coordinates and not the state vector,
it is more convenient to work in the embedding space $\mathbb{R}^{D}$
and in terms of $\omega$ and $\mu$. Therefore, we will rely on the
bijective correspondence between $X$ and $\tilde{X}$ and use the
notation $u(t;\tau;\omega)$ instead of $u(t;\tau;\tilde{\omega})$
and $u(t;\omega)$ instead of $u(t;\tilde{\omega})$. With these conventions,
$u(t;\tau;\omega)$ can be thought of as the path in $\mathbb{R}^{D}$
with $u(0;\tau;\omega)=\omega$. Similarly, $u(t;\omega)$ can be
thought of as a real-valued signal with $u(0;\omega)=\omega_{1}$,
where $\omega_{1}$ is the first component of $\omega\in\mathbb{R}^{D}$.
In later arguments, the assumption that $\omega$ is $\mu$-distributed
will be significant, and so will be the ergodicity of the flow with
respect to $\mu$. 

Given the signal $u(t;\omega)$, it is assumed that the recorded observations
are $u_{\eta}(jh;\omega)=u(jh;\omega)+\epsilon_{j}$, where $\epsilon_{j}$
is iid noise. Following Eggermont and LaRiccia \cite{EggermontLariccia2006,EggermontLaRiccia2009},
we assume that $\mathbb{E}\epsilon_{j}=0$ and $\mathbb{E}\abs{\epsilon_{j}}^{\kappa}<\infty$
for some $\kappa>3$. To avoid inessential technicalities it is assumed
that $\tau/h\in\mathbb{Z}^{+}$ so that the delay is an integral multiple
of the time step $h$. In particular, we set $\tau=nh$. Similarly,
we assume $t_{f}=n_{f}\tau$, $n_{f}\in\mathbb{Z}^{+}$, where $t_{f}$
is the look-ahead into the future. The noisy delay coordinates $u_{\eta}(jh;\tau;\omega)$
are assumed to be available for $j=0,\ldots,\left(N+n_{f}\right)n$,
which implies that the observation interval of $u_{\eta}(t;\omega)$
is $t\in[-(D-1)\tau,N\tau+t_{f}]$. 

The exact predictor $F:\mathbb{R}^{D}\rightarrow\mathbb{R}$ is a
$C^{r}$ function such that $F(u(t;\tau;\omega))=u(t+t_{f};\omega)$
for $\omega\in X$. Lemma \ref{lem:lemma-3} proves uniqueness and
existence of the exact predictor $F$. The exact predictor $F$ corresponds
to a fixed $t_{f}>0$, but that dependence is not shown in the notation.
The problem as considered by Müller et al \cite{MullerSmolaRatschScholkopfKohlmorgenVapnik1998}
is to recover the exact predictor $F$ from the noisy observations
$u_{\eta}(jh;\omega)$. Let $\abs{\cdot}_{\epsilon}$ denote Vapnik's
$\epsilon$-loss function. The algorithm of Müller et al computes
$f_{m}$ such that the functional
\begin{equation}
\frac{1}{Nn+1}\sum_{j=0}^{Nn}\abs{f(u_{\eta}(jh;\tau;\omega))-u_{\eta}(jh+\tau;\omega)}_{\epsilon}+\Lambda\norm f_{K_{\gamma}}^{2}\label{eq:algo-muller}
\end{equation}
 is minimized for $f=f_{m}$ in the reproducing kernel Hilbert space
$\mathcal{H}_{K_{\gamma}}$ corresponding to the kernel $K_{\gamma}$.
The kernel $K_{\gamma}$ is assumed to be given by $K_{\gamma}(x,y)=\exp\left(-\frac{\sum_{i=1}^{D}(x_{i}-y_{i})^{2}}{\gamma^{2}}\right)$.
The kernel bandwidth parameter $\gamma$ and the Lagrange multiplier
$\Lambda$ are both determined using cross-validation. This method
approximates the exact predictor $F$ for $t_{f}=\tau$. If $t_{f}=n_{f}\tau$,
$n_{f}\in\mathbb{Z}^{+}$, the approximation is iterated $n_{f}$
times. We will compare our predictor against that of Müller et al
using some of the same examples and the same framework as they do
in section 4.

In our algorithm, the first step is to apply spline smoothing. In
particular, we apply cubic spline smoothing \cite{deBoor2001} to
compute a function $u_{s}(t;\omega)$, $t\in[-(D-1)\tau,N\tau+t_{f}]$
such that the functional 
\begin{equation}
\frac{1}{(N+n_{f}+D-1)n+1}\sum_{j=-(D-1)n}^{(N+n_{f})n}\left(u_{\eta}(jh;\omega)-\tilde{u}(jh)\right)^{2}+\lambda\int_{-(D-1)\tau}^{N\tau+t_{f}}\tilde{u}''(t)^{2}\,dt\label{eq:algo-1}
\end{equation}
is minimum for $\tilde{u}=u_{s}(\cdot;\omega)$ over $\tilde{u}\in W^{2,2}[-(D-1)\tau,N\tau+t_{f}]$,
where $W^{2,2}[a,b]$ denotes the Sobolev space of twice-differentiable
functions $g:[a,b]\rightarrow\mathbb{R}$ with the norm $\norm g^{2}=\norm g_{2}^{2}+\norm{g'}_{2}^{2}+\norm{g''}_{2}^{2}$.
The parameter $\lambda$ is determined using five-fold cross-validation.
The minimizer $u_{s}(t;\omega)$ depends upon the noise-free signal
$u(t;\omega)$ as well as the instantiation of the iid noise in $u_{\eta}(jh;\omega)$
for $-(D-1)n\leq j\leq(N+n_{f})n$. However, the dependence on the
iid noise is not shown in the notation.

The second step of our algorithm is similar to the method of Müller
et al. The predictor $f_{1}$ is computed as 
\begin{equation}
f_{1}=\argmin_{f\in\mathcal{H}_{k}}\frac{1}{Nn+1}\sum_{j=0}^{Nn}\left(f(u_{s}(jh;\tau;\omega))-u_{s}(jh+t_{f};\omega)\right)^{2}+\Lambda\norm f_{K_{\gamma}}^{2}.\label{eq:algo-2}
\end{equation}
Both the parameters $\gamma$ and $\Lambda$ are determined using
five-fold cross-validation.\lyxdeleted{Divakar Viswanath,,,}{Fri Apr 20 21:27:34 2018}{
} Here $n_{f}$ and therefore $t_{f}$ are fixed because we seek to
approximate the exact predictor with lookahead fixed at $t_{f}$.
As explained in section 4, it is significant that the predictor directly
optimizes with a lookahead of $t_{f}$. Iterating a $\tau$-step predictor
$n_{f}$ times gives worse predictions.

The second step (\ref{eq:algo-2}) differs from the algorithm of Müller
et al in using the spline smoothed signal $u_{s}(t;\omega)$ in place
of the noisy signal $u_{\eta}(t;\omega)$. Our algorithm relies mainly
on spline smoothing to eliminate noise. Yet another difference is
that we use the least squares loss function in place of the $\epsilon$-loss
function. This difference is a consequence of relying on spline smoothing
to eliminate noise. As explained by Christmann and Steinwart \cite{ChristmannSteinwart2007},
the $\epsilon$-loss function, Huber's loss, and the $L^{1}$ loss
function are used to handle outliers. However, spline smoothing eliminates
outliers, and we choose the $L^{2}$ loss function because of its
algorithmic advantages.

We now turn to a discussion of the convergence of the predictor $f_{1}$
to the exact predictor $F$. The first step is to assess the accuracy
of spline smoothing. We quote the following lemma, which is a convenient
restatement of a result of Eggermont and LaRiccia \cite{EggermontLariccia2006,EggermontLaRiccia2009}
(see pages 132 and 133 of \cite{EggermontLaRiccia2009}). In the lemma,
$W^{m,2}[a,b]$ denotes the Sobolev space of $m$-times differentiable
functions $g:[a,b]\rightarrow\mathbb{R}$ with norm $\norm g^{2}=\sum_{j=0}^{m}\norm{g^{(j)}}_{2}^{2}$.
\begin{lem}
Assume $2\leq m\leq r$. Suppose that $u(t;\omega)$ is a signal defined
for $t\in\mathbb{R}$ with $\omega\in X$. For $j=-(D-1)n,\ldots,Nn+n_{f}$,
let $y_{j}=u(jh;\omega)+\epsilon_{j}$, where $h=\tau/n$ and where
$\epsilon_{j}$ are iid random variables. It is further assumed that
$\mathbb{E}\epsilon_{j}=0$, $\mathbb{E}\epsilon_{j}^{2}=\sigma^{2}$,
and $\mathbb{E}\abs{\epsilon_{j}}^{\kappa}<\infty$ for some $\kappa>3$.
Let $u_{s}(t)\in W^{m,2}[-(D-1)\tau,N\tau+t_{f}]$ be the spline that
minimizes the functional 
\[
\frac{1}{n(N+D-1)+n_{f}+1}\sum_{j=-(D-1)n}^{(N+n_{f})n}(\tilde{u}(jh)-y_{j})^{2}+\lambda\int_{-(D-1)\tau}^{N\tau+t_{f}}\abs{\tilde{u}^{(m)}(t)}^{2}\,dt
\]
 over $\tilde{u}\in W^{m,2}[-(D-1)\tau,N\tau]$. Assume
\[
\lambda=\left(\frac{\log(n(N+n_{f}+D-1))}{n(N+n_{f}+D-1)}\right)^{\frac{2m}{2m+1}}.
\]
 Let $p=\mathbb{P}\left(n,N,\Delta,\omega\right)$ be the probability
that
\[
\norm{u_{s}(\cdot;\omega)-u(\cdot;\omega)}_{\infty}>\Delta>0,
\]
where the $\infty$-norm is over the interval $[-(D-1)\tau,N\tau+t_{f}]$.
Then $\lim_{n\rightarrow\infty}\mathbb{P}(n,N,\Delta,\omega)=0$.\label{lem:Eggermont-LaRiccia}
\end{lem}
Some remarks about the connection of this lemma to the algorithm given
by (\ref{eq:algo-1}) and (\ref{eq:algo-2}) follow. First, the lemma
assumes a fixed choice of $\lambda$ (the relevant theorem in \cite{EggermontLariccia2006,EggermontLaRiccia2009}
in fact allows $\lambda$ to lie in an interval). In our algorithm,
$\lambda$ is determined using cross-validation because of its practical
effectiveness \cite{Wahba1990}. 

Second, the probability $\mathbb{P}(n,N,\Delta,\omega)$ (which may
be interpreted as the probability that spline smoothing fails to denoise
effectively) depends on $\omega$ and therefore on the particular
trajectory. If $\mathbb{P}(n,N,\Delta,\omega)$ depends on $\omega$
only though a bound on the $m$-th derivative of $u(t;\omega)$, $t\in[-(D-1)n,nN]$,
the bound would be uniform for all trajectories on the compact invariant
set $X$. The achievability part of Stone's optimality result \cite{Stone1982}
gives such a bound but the algorithm in that proof does not appear
practical. Proving a similar result for smooth splines based on the
existing literature does not appear entirely straightforward. In the
$L^{2}$ norm, some uniform bounds have been proved for smooth splines
by Györfi et al \cite{GyorfiKohlerKrzyzakWalk2002}. A bound on the
$L^{2}$ norm can be combined with a bound on the the $m$-th derivative
using a Sobolev inequality to obtain an $\infty$-norm bound. Although
the rate of convergence would be slightly sub-optimal, it would suffice
for our purposes. However, the result of Györfi et al is for expectations
and not for convergence in probability, and an argument using Chebyshev's
inequality does not give strong bounds.

The convergence analysis of the second half of the algorithm also
alters the algorithm slightly. In particular, the use of cross-validation
to choose parameters is not a part of the analysis. To state the convergence
theorem, we first fix $\epsilon>0$. By the universality theorem of
Steinwart \cite{Steinwart2001}, we may choose $F_{\epsilon}\in\mathcal{H}_{K_{\gamma}}$
such that $\norm{F_{\epsilon}-F}_{\infty}<\epsilon$ in a compact
domain that has a non-empty interior and contains the invariant set
$X$. The convergence theorem also makes the technical assumption
$\epsilon^{2}/\norm{F_{\epsilon}}_{K_{\gamma}}^{2}<1$, which may
always be satisfied by taking $\epsilon$ small enough.

The choice of the kernel-width parameter $\gamma$ is important in
practice. In the convergence proof, the choice of $\gamma$ is not
explicitly considered. However, $\gamma$ still plays a role because
$\norm{F_{\epsilon}}_{K_{\gamma}}$ depends upon $\gamma$. 

The parameter $\Lambda$ in (\ref{eq:algo-2}) is fixed as $\Lambda=\epsilon^{2}/\norm{F_{\epsilon}}_{K_{\gamma}}^{2}$for
the proof. Next we pick $\delta=\epsilon^{1/2}$ and $\ell\in\mathbb{Z}^{+}$
such that the covering of the invariant set $X$ using boxes of dimension
$2^{-\ell}$ ensures that the variation of $F_{\epsilon}$ (as well
as that of the exact predictor $F$ and $f_{3}$, which is defined
later) within each box is bounded by $\delta/4$.

Suppose $A_{1},\ldots,A_{L}$ are boxes of dimension $2^{-\ell}$
that cover $X$ in the manner hinted above. We next choose $T^{\ast}$
such that the measure of the trajectories (with respect to the ergodic
measure $\mu$) that sample each one of the boxes $A_{j}$ adequately
(in a sense that will be explained) is greater than $1-\epsilon$
if the time interval of the trajectory exceeds $T^{\ast}$.

The parameter $\Delta$ is a bound on the infinite norm accuracy of
the smooth spline as in Lemma \ref{lem:Eggermont-LaRiccia}. Choose
$\Delta>0$ small enough that 
\[
\frac{B_{1}\Delta^{1/2}}{\Lambda}=\frac{B_{1}\Delta^{1/2}\norm{F_{\epsilon}}_{K}^{2}}{\epsilon^{2}}<\epsilon^{1/2},
\]
where $B_{1}$ is a constant specified later. The main purpose of
increasing $n$ is to make spline smoothing accurate. However, the
following condition requiring $n$ to be large enough is assumed in
the proof: 
\[
\frac{B_{1}h^{1/2}}{\Lambda}=\frac{B_{1}\tau^{1/2}\norm{F_{\epsilon}}_{K}^{2}}{\epsilon^{2}n^{1/2}}<\epsilon^{1/2}.
\]
Within this set-up, we have the following convergence theorem.
\begin{thm}
For $\epsilon>0$, $T>T^{\ast}$, $N=T/\tau$, and $\Lambda$, $\Delta$
chosen as above, we have 
\[
\mu\left\{ x\in X\biggl|\abs{f_{1}(x)-F(x)}>3\sqrt{\epsilon}\right\} <\frac{8\epsilon}{1-\epsilon},
\]
when $f_{1}$ is constructed (or learnt) from the signal $u_{\eta}(t;\omega)$,
$t\in[-(D-1)\tau,N\tau]$, for $\left\{ \omega\in X\right\} $ of
$\mu$-measure greater than $1-\epsilon$ and with probability $1-\mathbb{P}(n,N,\Delta,\omega)$
(probability of successful denoising in the spline-smoothing step)
tending to $1$ in the limit $n\rightarrow\infty$. \label{thm:convergence-theorem-main}
\end{thm}
Nonuniform bounds implying a form of weak consistency are considered
by Steinwart, Hush, and Scovel \cite{SteinwartHushScovel2009}. However,
the algorithm of (\ref{eq:algo-1}) and (\ref{eq:algo-2}) does not
fit into the framework of \cite{SteinwartHushScovel2009}. The application
of spline smoothing to produce $u_{s}(t;\omega)$ means that $u_{s}(t;\omega)$
may not be stationary, and our method of analysis does not rely on
verifying a weak law of large numbers as in \cite{SteinwartHushScovel2009}.
The analysis summarized above and given in detail in the following
section relies on $\infty$-norm bounds.

\section{Proof of convergence}

We begin the proof with a more complete account of how the embedding
theorem is applied. Let $\frac{dU}{dt}=\mathcal{F}(U)$, where $\mathcal{F}\in C^{r}(\mathbb{R}^{d},\mathbb{R}^{d})$,
$r\geq2$, be a flow. Let $\mathcal{F}_{t}(U_{0})$ be the time-$t$
map with initial data $U_{0}$. Let $\tilde{V}\subset\mathbb{R}^{d}$
be an open set with compact closure. If $U_{0}\in\tilde{V}$, it is
assumed that $\mathcal{F}_{t}(U_{0})$ is well-defined for $-\tau D\leq t\leq n_{f}\tau=t_{f}$,
where $D$ is the embedding dimension.

Assumption: For embedding dimension $D\geq2d+1$ and a suitably chosen
delay $\tau>0$, the map
\[
x\rightarrow(\pi x,\pi\mathcal{F}_{-\tau}x,\pi\mathcal{F}_{-2\tau}x,\ldots,\pi\mathcal{F}_{-(D-1)\tau}x)
\]
is a $C^{r}$ diffeomorphism between $\tilde{V}$ and its image in
$\mathbb{R}^{D}$. This assumption is generically true \cite{SauerYorkeCasdagli1991}.
This map is called the delay embedding. Denote the image of $\tilde{V}$
under the delay embedding by $V$. 

The invariant measures $\tilde{\mu}$ and $\mu$ as well as $\tilde{X}$,
$X$, $\tilde{\omega}$, $\omega$, $u(t;\omega)$, and $u(t;\tau;\omega)$
are as defined earlier. It is assumed that $\tilde{X}\subset\tilde{V}$,
which implies $X\subset V$.
\begin{lem}
\label{lem:lemma-3}Suppose $\frac{dU(t)}{dt}=\mathcal{F}(U(t))$
for $-\tau D\leq t\leq t_{f}$, $U(0)=U_{0}\in\tilde{V}$. Denote
the delay vector 
\[
\left(\pi U_{0},\pi\mathcal{F}_{-\tau}U_{0},\ldots,\pi\mathcal{F}_{-(D-1)\tau}U_{0}\right)
\]
by $U_{0,\tau}$ so that $U_{0,\tau}\in V$. There exists a unique
and well-defined $C^{r}$ function $F:V\rightarrow\mathbb{R},$ called
the exact predictor, such that 
\[
F(U_{0,\tau})=\pi\mathcal{F}_{t_{f}}(U_{0})
\]
for all $U_{0,\tau}\in V$. In particular, $F(u(t;\tau;\omega))=u(t+t_{f};\omega)$
for all $t\in\mathbb{R}$ and all $\omega\in X$.\end{lem}
\begin{proof}
To map $U_{0,\tau}\in V$ to $\pi\mathcal{F}_{t_{f}}(U_{0})$, first
invert the delay map to obtain the point $U_{0}$ in $\tilde{V}$,
advance that point by $t_{f}$ by applying $\mathcal{F}_{t_{f}}$,
and finally project using $\pi$. Each of the three maps in this composition
is $C^{r}$ or better. The predictor must be unique because $\mathcal{F}_{t_{f}}$
is uniquely determined by the flow.\end{proof}
\begin{rem*}
The embedding theory of Sauer et al \cite{SauerYorkeCasdagli1991}
may be applied to the compact invariant set $\tilde{X}$ without enclosing
it in the open set $\tilde{V}$. Indeed, if the box counting dimension
of $\tilde{X}$ is $d'$, the embedding dimension need only satisfy
$D\in\mathbb{Z}^{+}$ and $D>2d'$. That can be advantageous because
we may have $d'$ much smaller than $d$. However, there are two difficulties
if $\tilde{X}$ is a fractal set. First, tangent spaces cannot be
defined and we cannot assert the delay map to be a diffeomorphism
although it will be one-one generically. Second, we will need to extend
$F$ to the closure of an open neighborhood of $X$ in $\mathbb{R}^{D}$
to apply the universality theorem, and such an extension cannot be
made from $X$ if $X$ is a fractal set. Both these difficulties go
away if we take $\tilde{V}$ to be a submanifold that contains $\tilde{X}$.
If $d'$ is the dimension of $\tilde{V}$, we would only require $D>2d'$.
For simplicity, we have assumed $\tilde{V}$ to be an open set.
\end{rem*}
The following convexity lemma is an elementary result of convex analysis
\cite{EkelandTemam1987}. It is stated and proved for completeness.
\begin{lem}
Let $\mathcal{L}_{1}(f)$ and $\mathcal{L}_{2}(f)$ be convex and
continuous in $f$, where $f\in\mathcal{H}$ and $\mathcal{H}$ is
a Hilbert space. If $w\in\nabla\mathcal{L}_{i}(f)$, the subgradient
at $f$, assume that 
\[
\mathcal{L}_{i}(f+g)-\mathcal{L}_{i}(f)-\inner wg\geq\lambda\inner gg/2
\]
 for $\lambda>0$, all $g\in\mathcal{H}$, and $i=1,2$. Let $f_{1}=\argmin\mathcal{L}_{1}(f)$
and $f_{2}=\argmin\mathcal{L}_{2}(f)$. Suppose that 
\[
\abs{\mathcal{L}_{1}(f)-\mathcal{L}_{2}(f)}\leq\delta
\]
for $\norm f\leq r$, and assume that $\norm{f_{1}}<r$ and $\norm{f_{2}}<r$.
Then, 
\[
\norm{f_{1}-f_{2}}^{2}\leq\frac{2\delta}{\lambda}.
\]
\label{lem:convexity-bound}\end{lem}
\begin{proof}
Because $f_{1}$ minimizes $\mathcal{L}_{1}(f)$, we have $0\in\nabla\mathcal{L}_{1}(f_{1})$.
Thus, 
\[
\mathcal{L}_{1}(f_{2})-\mathcal{L}_{1}(f_{1})\geq\lambda\norm{f_{2}-f_{1}}^{2}/2.
\]
Similarly, $\mathcal{L}_{2}(f_{1})-\mathcal{L}_{2}(f_{2})\geq\lambda\norm{f_{2}-f_{1}}^{2}/2$.
By adding the two inequalities, we have 
\[
\norm{f_{2}-f_{1}}^{2}\leq\frac{\abs{\mathcal{L}_{1}(f_{2})-\mathcal{L}_{1}(f_{1})+\mathcal{L}_{2}(f_{1})-\mathcal{L}_{2}(f_{2})}}{\lambda}\leq\frac{2\delta}{\lambda},
\]
proving the lemma. This last step relies on $\norm{\mathcal{L}_{i}(f_{1})-\mathcal{L}_{i}(f_{2})}\leq\delta$
and the assumption $\norm{f_{1}},\norm{f_{2}}<r$.
\end{proof}
If $u(t;\omega)$, $t\in[-(D-1)\tau,N\tau+t_{f}]$, is the noise-free
signal, our arguments are phrased under the assumption that $\abs{u(t;\omega)-u_{s}(t;\omega)}\leq\Delta$.
This assumption is realized with probability $1-\mathbb{P}(n,N,\Delta,\omega)$,
which tends to $1$ as $n$ increases (by Lemma \ref{lem:Eggermont-LaRiccia}).
For convenience, we denote $\mathbb{P}(n,N,\Delta,\omega)$ by $p$.
The probability that $u_{\eta}(t;\omega)$ is successfully denoised
by smooth splines so that $\abs{u(t;\omega)-u_{s}(t;\omega)}\leq\Delta$
is then $1-p$.

In general, a $C^{r}$ function defined on an embedded submanifold
can be extended to an open neighborhood of the submanifold using a
partition of unity. Because $V\subset\mathbb{R}^{D}$ is an embedded
submanifold, $X\subset V$, and the exact predictor $F$ is defined
on $V$, it follows that there exists $M>0$ such that $F$ can be
extended to $Y$, where 
\[
Y=\left\{ y|\norm{y-\omega}_{\infty}\leq M\quad\text{for some}\quad\omega\in X\right\} .
\]
We will always assume $\Delta<M$ so that the spline-smoothed signal
maps to $Y$ under delay embedding with probability greater than $1-p$.
Without loss of generality, we assume $M\leq1$. The convergence proof
will assess the approximation to $F$ with respect to the measure
$\mu$. Therefore, the manner in which the extension is carried out
is not highly relevant. The sole purpose of the extension is to facilitate
an application of the universality theorem for Gaussian kernels.

Let 
\begin{equation}
B=\sup_{\omega\in X}\norm{\omega}_{\infty}+M\label{eq:defn-B-1}
\end{equation}
Thus, $B$ is a bound on the size of the embedded invariant set with
ample allowance for error in spline smoothing. 

Let $u_{s}(t;\omega)$ denote the spline-smoothed signal and $u(t;\omega)$
the noise-free signal with $\omega\in X$. Define 
\[
\mathcal{W}_{1}(f)=\frac{1}{Nn+1}\sum_{j=0}^{Nn}\left(f(u_{s}(jh;\tau;\omega))-u_{s}(jh+t_{f};\omega)\right)^{2}+\Lambda\norm f_{K}^{2},
\]
where $t_{f}=n_{f}\tau$, $n_{f}\in\mathbb{Z}^{+}$, and $K$ is any
smooth and positive kernel defined over $Y\times Y$. The kernel $K$
will be specialized to the Gaussian kernel $K_{\gamma}$ when applying
the universality theorem. Define 
\[
\mathcal{W}_{2}(f)=\frac{1}{Nn+1}\sum_{j=0}^{Nn}\left(f(u(jh;\tau;\omega))-u(jh+t_{f};\omega)\right)^{2}+\Lambda\norm f_{K}^{2}
\]
using the noise-free signal $u(t;\omega)$. Let $T=N\tau$ and define
\[
\mathcal{W}_{3}(f)=\frac{1}{T}\int_{0}^{T}\left(f(u(t;\tau;\omega))-u(t+t_{f};\omega)\right)^{2}\,dt+\Lambda\norm f_{K}^{2}.
\]
For $\Lambda>0$, all three functionals are strictly convex and have
a unique minimizer. The unique minimizers of $\mathcal{W}_{1}$, $\mathcal{W}_{2}$,
and $\mathcal{W}_{3}$ are denoted by $f_{1}$, $f_{2}$, and $f_{3}$,
respectively. The functional $\mathcal{W}_{1}$ is the same as in
(\ref{eq:algo-2}), the second step of the algorithm. Thus, $f_{1}$
is the computed approximation to the exact predictor $F$.

The following lemma bounds the minimizers of $\mathcal{W}_{1}(f)$,
$\mathcal{W}_{2}(f)$, $\mathcal{W}_{3}(f)$ in norm by $B/\Lambda^{1/2}$. 
\begin{lem}
The minimizer $f_{1}$ satisfies $\norm{f_{1}}_{K}\leq\frac{B}{\Lambda^{1/2}}$
with probability greater than $1-p$. The minimizers $f_{2}$ and
$f_{3}$ satisfy $\norm{f_{2}}_{K}\leq\frac{B}{\Lambda^{1/2}}$ and
$\norm{f_{3}}_{K}\leq\frac{B}{\Lambda^{1/2}}$.\label{lem:lemma-f1f2f3bound}\end{lem}
\begin{proof}
Because $f_{1}$ minimizes $\mathcal{W}_{1}(f)$, we must have $\mathcal{W}_{1}(f_{1})\leq\mathcal{W}_{1}(0)$.
We have $\mathcal{W}_{1}(0)\leq B^{2}$ with probability greater than
$1-p$. Thus, $\Lambda\norm{f_{1}}_{K}^{2}\leq\mathcal{W}_{1}(f_{1})\leq\mathcal{W}_{1}(0)\leq B^{2}$
and the stated bound for $\norm{f_{1}}_{K}$ follows. The bounds for
$f_{2}$ and $f_{3}$ are proved similarly. \end{proof}
\begin{lem}
Assume $0<\Lambda\leq1$ and $\abs{u(t;\omega)-u_{s}(t;\omega)}\leq\Delta$
for $t\in[-(D-1)\tau,T]$. For $f\in\mathcal{H}_{K}$ with $\norm f_{K}\leq\frac{B}{\Lambda^{1/2}}$,
we have $\abs{\mathcal{W}_{1}(f)-\mathcal{W}_{2}(f)}\leq\frac{B_{1}^{2}\Delta}{\Lambda}$.
Here $B_{1}$ depends only on $B$ and the kernel $K$. The kernel
$K$ is assumed to be $C^{2}$.\label{lem:W1-minus-W2}\end{lem}
\begin{proof}
First, we note that $\norm f_{\infty}\leq c_{0}\norm f_{K}$ and $\norm{\partial f}_{\infty}\leq c_{1}\norm f_{K}$,
where $\partial$ is the directional derivative of $f$ in any direction.
By a result of Zhou (part (c) of Theorem 1 of \cite{Zhou2008}), we
may take $c_{0}=\norm{K(x,y)}_{\infty}$ and $c_{1}D^{-1/2}=\norm{K(x,y)}_{\infty}+\sum\norm{\partial_{x_{i}}K(x,y)}_{\infty}+\sum\norm{\partial_{x_{i}}\partial_{x_{j}}K(x,y)}_{\infty}$,
where $D$ is the embedding dimension and the $\infty$-norm is over
$x,y\in Y$. If we define $B_{1}'$ using 
\begin{equation}
B_{1}'=\max(B,c_{0}B,c_{1}B),\label{eq:defn-B-2}
\end{equation}
it follows that both $\norm f_{\infty}$ and $\norm{\partial f}_{\infty}$
(where $\partial$ is a directional derivative in any direction) are
bounded above by $B_{1}'/\Lambda^{1/2}$ .

We may write 
\begin{equation}
\abs{\mathcal{W}_{1}(f)-\mathcal{W}_{2}(f)}\leq\frac{1}{Nn+1}\sum_{j=0}^{Nn}\frac{4B_{1}'}{\Lambda^{1/2}}\left(\begin{split}\abs{f(u_{s}(jh;\tau;\omega))-f(u(jh;\tau;\omega))}\\
+\abs{u_{s}(jh;\tau;\omega)-u(jh;\tau;\omega)}
\end{split}
\right).\label{eq:W1minusW2}
\end{equation}
Here $\frac{4B_{1}'}{\Lambda^{1/2}}$ is used an upper bound on $\abs{f(u_{s}(jh;\tau))}+\abs{f(u(jh;\tau))}+\abs{u_{s}(jh;\tau)}+\abs{u(jh;\tau)}$.
The bound of $B_{1}'/\Lambda^{1/2}$ on $\abs f$ is justified by
the previous paragraph. The same bound on $\abs{u_{s}}$ and $\abs u$
follows from $B_{1}'<B$ and $\Lambda\leq1$.

Now, $\abs{u_{s}(jh;\tau;\omega)-u(jh;\tau;\omega)}\leq\Delta$ implies
that 
\[
\abs{f(u_{s}(jh;\tau;\omega))-f(u(jh;\tau;\omega))}\leq B_{1}'\Delta/\Lambda^{1/2}
\]
 by the bound on $\norm{\partial f}_{\infty}$. By replacing $B_{1}'$
with $\max(B_{1}',1)$ if necessary, we have 
\[
\abs{u_{s}(jh;\tau;\omega)-u(jh;\tau;\omega)}\leq\Delta\leq B_{1}'\Delta/\Lambda^{1/2}.
\]
The proof is completed by utilizing these bounds in (\ref{eq:W1minusW2})
and defining $B_{1}$ as $B_{1}=\sqrt{8}B_{1}'$.\end{proof}
\begin{lem}
Assume $0<\Lambda\leq1$. With probability greater than $1-p$, $\norm{f_{1}-f_{2}}_{K}\leq\frac{B_{1}\Delta^{1/2}}{\Lambda}$.
\label{lem:f1-minus-f2}\end{lem}
\begin{proof}
Follows from Lemmas \ref{lem:lemma-f1f2f3bound}, \ref{lem:W1-minus-W2},
and \ref{lem:convexity-bound}. Lemma \ref{lem:convexity-bound} is
applied with $r=\frac{B}{\Lambda^{1/2}}$, $\delta=\frac{B_{1}^{2}\Delta}{\Lambda}$,
and $\lambda=2\Lambda$. The choice of $r$ is justified by Lemma
\ref{lem:lemma-f1f2f3bound} and the choice of $\delta$ is justified
by Lemma \ref{lem:W1-minus-W2}. To justify the choice of $\lambda$,
note that $\mathcal{W}_{1}(f)$ and $\mathcal{W}_{2}(f)$ can both
be written as $\mathcal{W}_{i}(f)=\mathcal{L}(f)+\Lambda\norm f_{K}^{2}$
with $\mathcal{L}$ a convex functional. The identity $\inner{f+g}{f+g}_{K}=\inner ff_{K}+\inner{2f}g_{K}+\inner gg_{K}$
shows that $2f$ is the unique subgradient at $f$ for $\Lambda\norm f_{K}^{2}$.
Thus, if $w\in\nabla\mathcal{W}_{i}(f)$ (the subgradient of $\mathcal{W}_{i}$
is unique and may be obtained explicitly), we must have $\mathcal{W}_{i}(f+g)-\mathcal{W}_{i}(f)-\inner wg_{K}\geq\Lambda\inner gg_{K}$,
justifying the choice of $\lambda$.\end{proof}
\begin{lem}
Assume $0\leq\Lambda\leq1$. For $f\in\mathcal{H}_{K}$ and $\norm f_{K}\leq\frac{B}{\Lambda^{1/2}}$,
we have $\abs{\mathcal{W}_{2}(f)-\mathcal{W}_{3}(f)}\leq\frac{B_{1}^{2}h}{\Lambda}.$\label{lem:W2-minus-W3}\end{lem}
\begin{proof}
We will argue as in Lemma \ref{lem:W1-minus-W2} and assume that $\norm f_{\infty}$,
and $\norm{\partial f}_{\infty}$ are bounded by $B_{1}'/\Lambda^{1/2}$.

Suppose $\alpha\in[0,1]$. In the difference 
\[
\begin{split}\frac{1}{h}\int_{kh}^{(k+1)h}\left(f(u(t;\tau;\omega))-u(t+t_{f};\omega)\right)^{2}\,dt & -(1-\alpha)\left(f(u(kh;\tau;\omega))-u(kh+t_{f};\omega)\right)^{2}\\
 & -\alpha\left(f(u((k+1)h;\tau;\omega))-u((k+1)h+t_{f};\omega)\right)^{2},
\end{split}
\]
we may apply the mean value theorem to the integral and argue as in
Lemma \ref{lem:W1-minus-W2} to upper bound the difference by $\left(B_{1}'\right)^{2}h/\Lambda$.
The proof is completed by summing the differences from $k=0$ to $k=Nn-1$
and dividing by $Nn$. \end{proof}
\begin{lem}
Assume $0\leq\Lambda<1$. Then $\norm{f_{2}-f_{3}}\leq\frac{B_{1}h^{1/2}}{\Lambda}$.\label{lem:f2-minus-f3}\end{lem}
\begin{proof}
Follows from Lemmas \ref{lem:lemma-f1f2f3bound}, \ref{lem:W2-minus-W3},
and \ref{lem:convexity-bound}. Lemma \ref{lem:convexity-bound} is
applied with $r=\frac{B}{\Lambda^{1/2}}$, $\delta=\frac{B_{1}^{2}h}{\Lambda}$,
and $\lambda=2\Lambda$. The choices of $r$, $\delta$, and $\Lambda$
are justified using Lemmas \ref{lem:lemma-f1f2f3bound} and \ref{lem:W2-minus-W3}
and an additional argument as in the proof of Lemma \ref{lem:f1-minus-f2}.
\end{proof}
Choose $\epsilon>0$. At this point, we specialize $K$ to a kernel
for which the universality theorem of Steinwart applies. For example,
$K=K_{\gamma}$. We may then find $F_{\epsilon}\in\mathcal{H}_{K}$
such that $\norm{F_{\epsilon}-F}_{\infty}\leq\epsilon$, where the
$\infty$-norm is over $Y$. In fact, we will need the difference
$\abs{F_{\epsilon}(x)-F(x)}$ to be bounded by $\epsilon$ only for
$x\in X$. The larger compact space $Y$ is needed to apply the universality
theorem and for other RKHS arguments. 
\begin{lem}
Let $\Lambda=\epsilon^{2}/\norm{F_{\epsilon}}_{K}^{2}\leq1$. If $f_{3}$
minimizes $\mathcal{W}_{3}(f)$, we have 
\[
\frac{1}{T}\int_{0}^{T}\left(f_{3}(u(t;\tau;\omega))-u(t+t_{f};\omega)\right)^{2}\,dt\leq\Lambda\norm{F_{\epsilon}}_{K}^{2}+\epsilon^{2}=2\epsilon^{2}.
\]
 In addition, $\norm{f_{3}}_{K}^{2}\leq2\norm{F_{\epsilon}}_{K}^{2}.$\label{lem:choice-of-Lambda}\end{lem}
\begin{proof}
We have 
\[
\frac{1}{T}\int_{0}^{T}\left(f_{3}(u(t;\tau;\omega))-u(t+t_{f};\omega)\right)^{2}\,dt\leq\mathcal{W}_{3}(f_{3}),
\]
$\mathcal{W}_{3}(f_{3})\leq\mathcal{W}_{3}(F_{\epsilon})$ because
$f_{3}$ is the minimizer, and 
\[
\mathcal{W}_{3}(F_{\epsilon})\leq\epsilon^{2}+\Lambda\norm{F_{\epsilon}}_{K}^{2}.
\]
This last inequality uses $\int(F_{\epsilon}(u(t;\tau;\omega))-u(t+t_{f};\omega))^{2}\,dt=\int(F_{\epsilon}(u(t;\tau;\omega))-F(u(t;\tau;\omega))^{2}\,dt$.
The proof of the first part of the lemma is completed by combining
the inequalities. To prove the second part, we argue similarly after
noting $\norm{f_{3}}_{K}^{2}\leq\mathcal{W}_{3}(F_{\epsilon})/\Lambda$.
\end{proof}
Consider half-open boxes in $\mathbb{R}^{D}$ of the form 
\[
A_{j_{1},j_{2},\ldots,j_{D}}=\halfi{\frac{j_{1}}{2^{\ell}},\frac{j_{1}+1}{2^{\ell}}}\times\cdots\times\halfi{\frac{j_{D}}{2^{\ell}},\frac{j_{D}+1}{2^{\ell}}},
\]
with $\ell\in\mathbb{Z}^{+}$and $j_{i}\in\mathbb{Z}$. The whole
of $\mathbb{R}^{D}$ is a disjoint union of such boxes. Because $X$
is compact, we can assume that $X\subset\cup_{j=1}^{L}A_{j}$, where
the union is disjoint, each $A_{j}$ is a half-open box of the form
above, and $A_{j}\cap X\neq\phi$ for $1\leq j\leq L$. 

We will pick $\ell$ to be so large, that each box has a diameter
that is bounded as follows:
\[
\frac{\sqrt{D}}{2^{\ell}}<\frac{\delta}{4\sqrt{2}D^{1/2}\norm{\partial^{2}K}_{2,\infty}^{1/2}\norm{F_{\epsilon}}_{K}}.
\]
Here $\delta>0$ is determined later, and $\norm{\partial^{2}K}_{2,\infty}$
is the $\infty$-norm in the function space $C^{2}(Y\times Y)$. Lemma
\ref{lem:choice-of-Lambda} tells us that $\norm{f_{3}}_{K}\leq\sqrt{2}\norm{F_{\epsilon}}_{K}$,
and therefore (by part (c) of Theorem 1 of \cite{Zhou2008}) $\norm{\partial f_{3}}_{\infty}\leq\sqrt{2}D^{1/2}\norm{\partial^{2}K}_{2,\infty}^{1/2}\norm{F_{\epsilon}}_{K}$.
As a consequence of our choice of $\ell$, $x,y\in A_{j}$ implies
that 
\begin{equation}
\abs{f_{3}(x)-f_{3}(y)}<\delta/4,\label{eq:f3-derv-bound}
\end{equation}
bounding the variation of $f_{3}$ within a single cell $A_{j}$.
Because the exact predictor $F$ is $C^{r}$, $r\geq2$, and $X$
is compact, we may also assert that 
\begin{equation}
\abs{F(x)-F(y)}<\delta/4\label{eq:F-derv-bound}
\end{equation}
for $x,y\in A_{j}$ by taking $\ell$ larger if necessary.

The next lemma is about taking a trajectory that is long enough that
each of the sets $A_{j}$ is sampled accurately. By assumption $X$
is the support of $\mu$. However, we may still have $\mu(A_{j})=0$
for some $j$. In the following lemma and later, it is assumed that
all $A_{j}$ with $\mu(A_{j})=0$ are eliminated from the list of
boxes covering $X$.
\begin{lem}
Let $\chi_{A_{j}}$ denote the characteristic function of the set
$A_{j}$. There exist $T^{\ast}>0$ and a Borel measurable set 
\[
S_{\epsilon,T^{\ast}}\subset X
\]
such that $\omega\in S_{\epsilon,T^{\ast}}$ implies that for all
$T\geq T^{\ast}$ and $j=1,\ldots,L$ 
\[
\abs{\frac{1}{T}\int_{0}^{T}\chi_{A_{j}}\left(u(t;\tau;\omega)\right)\,dt-\mu(A_{j})}\leq\epsilon\mu(A_{j}).
\]
and with $\mu\left(S_{\epsilon,T^{\ast}}\right)>1-\epsilon$.\label{lem:covering-lemma}\end{lem}
\begin{proof}
To begin with, consider the set $A_{1}$. By the ergodic theorem,
\[
\lim_{T\rightarrow\infty}\frac{1}{T}\int_{0}^{T}\chi_{A_{1}}\left(u(t;\tau;\omega)\right)\,dt=\mu(A_{1})
\]
for $\omega\in S\subset X$ with $\mu(S)=1$. Let 
\[
A_{s,\epsilon}=\left\{ \omega\in X\biggl|\abs{\frac{1}{T}\int_{0}^{T}\chi_{A_{1}}(u(t;\tau;\omega))\,dt-\mu(A_{1})}>\epsilon\mu(A_{1})\:\:\text{for some}\:\:T\geq s\right\} .
\]
The sets $A_{s,\epsilon}$ shrink with increasing $s$. Then the measure
of $\cap_{s=1}^{\infty}A_{s,\epsilon}$ under $\mu$ is zero. Therefore,
there exists $s_{1}\in\mathbb{Z}^{+}$ such that $\mu(A_{s_{1},\epsilon})<\epsilon/L$. 

We can find $s_{2},\ldots,s_{L}$ similarly by considering the sets
$A_{2},\ldots,A_{L}$. The lemma then holds with $T^{\ast}=\max(s_{1},\ldots,s_{L})$.\end{proof}
\begin{lem}
Suppose that $\omega\in S_{\epsilon,T^{\ast}}$, $T\geq T^{\ast}$,
and $\Lambda=\epsilon^{2}/\norm{F_{\epsilon}}_{K}^{2}\leq1$. Suppose
that $f_{3}$ minimizes $\mathcal{W}_{3}(f)$, which is defined using
$u(t;\omega)$, $T$, and $\Lambda$ . Then 
\[
\mu\left\{ x\in X\bigl|\abs{f_{3}(x)-F(x)}\geq\delta\right\} <\frac{8\epsilon^{2}}{\delta^{2}(1-\epsilon)}.
\]
\label{lem:cvg-lemma}\end{lem}
\begin{proof}
Denote the set $\left\{ x\in X\bigl|\abs{f_{3}(x)-F(x)}\geq\delta\right\} $
by $S_{\delta}$. Let $J$ be the set of all $j=1,\ldots,L$ such
that $\abs{f_{3}(x)-F(x)}\geq\delta$ for some $x\in A_{j}$. Evidently,
$S_{\delta}\subset\cup_{j\in J}A_{j},$ and it is sufficient to bound
the measure of $\cup_{j\in J}A_{j}$.

By (\ref{eq:f3-derv-bound}) and (\ref{eq:F-derv-bound}), if$\abs{f_{3}(x)-F(x)}\geq\delta$
for some $x\in A_{j}$ then for any $y\in A_{j}$, we have 
\begin{align}
\abs{f_{3}(y)-F(y)} & \geq\abs{f_{3}(x)-F(x)}-\abs{f_{3}(x)-f_{3}(y)}-\abs{F(x)-F(y)}\nonumber \\
 & >\frac{\delta}{2}.\label{eq:f3-minus-F}
\end{align}
For $\omega\in S_{\epsilon,T^{\ast}}$, we have 
\begin{align*}
\frac{1}{T}\int_{0}^{T}(f_{3}(u(t;\tau;\omega)-u(t+t_{f};\omega))^{2}\,dt & =\frac{1}{T}\int_{0}^{T}(f_{3}(u(t;\tau;\omega)-F(u(t;\tau;\omega))^{2}\,dt\\
 & \geq\frac{1}{T}\int_{0}^{T}(f_{3}(u(t;\tau;\omega)-F(u(t;\tau;\omega))^{2}\sum_{j\in J}\chi_{A_{j}}\left(u(t;\tau;\omega)\right)\,dt\\
 & =\frac{1}{T}\sum_{j\in J}\int_{0}^{T}(f_{3}(u(t;\tau;\omega))-F(u(t;\tau;\omega))^{2}\chi_{A_{j}}\left(u(t;\tau;\omega)\right)\,dt\\
 & \geq\frac{\delta^{2}}{4T}\sum_{j\in J}\int_{0}^{T}\chi_{A_{j}}\left(u(t;\tau;\omega)\right)\,dt\\
 & \geq\frac{\delta^{2}}{4}\mu\left(\cup_{j\in J}A_{j}\right)(1-\epsilon),
\end{align*}
where the first inequality holds because $A_{j}$ are disjoint, the
second inequality holds because $\abs{f_{3}(y)-F(y)}>\delta/2$ follows
from (\ref{eq:f3-minus-F}) for $y=u(t;\tau;\omega)\in A_{j}$ with
$j\in J$, and the final inequality is a consequence of Lemma \ref{lem:covering-lemma}
and $\omega\in S_{\epsilon,T^{\ast}}$.

Applying Lemma \ref{lem:choice-of-Lambda}, we get 
\[
\frac{\delta^{2}}{4}\mu\left(\bigcup_{j\in J}A_{j}\right)(1-\epsilon)\leq2\epsilon^{2},
\]
completing the proof of the lemma.\end{proof}
\begin{lem}
Suppose $\omega\in S_{\epsilon,T^{\ast}}$ and that the signals $u(t;\omega)$
and $u_{\eta}(t;\omega)$ are used to define $\mathcal{W}_{i}(f)$,
$i=1,2,3$. Suppose that $f_{1}$, $f_{2}$, and $f_{3}$ minimize
$\mathcal{W}_{1}(f)$, $\mathcal{W}_{2}(f)$, and $\mathcal{W}_{3}(f)$,
respectively, with $T\geq T^{\ast}$ and $\Lambda=\epsilon^{2}/\norm{F_{\epsilon}}_{K}^{2}\leq1$.
Then 
\[
\mu\left\{ x\in X\biggl|\abs{f_{1}(x)-F(x)}>\delta+\frac{B_{1}h^{^{1/2}}+B_{1}\Delta^{1/2}}{\Lambda}\right\} <\frac{8\epsilon^{2}}{\delta^{2}(1-\epsilon)}
\]
 with probability greater than $1-p$.\end{lem}
\begin{proof}
Follows from Lemmas \ref{lem:f1-minus-f2}, \ref{lem:f2-minus-f3},
and \ref{lem:cvg-lemma}.
\end{proof}
The above lemma implies Theorem \ref{thm:convergence-theorem-main}
with the choice of $\delta$, $n$, and $\Delta$ specified above
it.

\section{Numerical illustrations}

We compare three methods to compute an approximate predictor $f$.
The first method is that of Müller et al \cite{MullerSmolaRatschScholkopfKohlmorgenVapnik1998}
given in (\ref{eq:algo-muller}). The second method is exactly the
same but with the least squares regression function. The third method
is the convergent algorithm given by (\ref{eq:algo-1}) and (\ref{eq:algo-2}). 

When comparing the methods, we always used the same noisy data for
all three methods. There can be some fluctuation due to the instance
of noise that is added to the exact signal $\tilde{x}(t)$ as well
as the segment of signal that is used. The effect of this fluctuation
on comparison is eliminated by using the same noisy data in each case.
In addition, reported results are averages over multiple datasets.
For all three methods, the error in the approximate predictor is estimated
by applying it to a noise-free stretch of the signal as in \cite{MullerSmolaRatschScholkopfKohlmorgenVapnik1998},
which is standard because the object of each method is to approximate
the exact predictor. 

The first signal we use is the same as in \cite{MullerSmolaRatschScholkopfKohlmorgenVapnik1998},
except for inevitable differences in instantiation. The Mackey-Glass
equation 
\[
\frac{d\tilde{x}(t)}{dt}=-0.1\tilde{x}(t)+\frac{0.2\tilde{x}(t-D)}{1+\tilde{x}(t-D)^{10}},
\]
with $D=17$, is solved with time step $\Delta t=0.1$ and transients
are eliminated to produce the exact signal $\tilde{x}(t)$. This signal
will of course have rounding errors and discretization errors, but
those are negligible compared to prediction errors. The standard deviation
of the Mackey-Glass signal is about $0.23$. An independent normally
distributed quantity of mean zero is added at each point so that the
ratio of the variance of the noise to that of the signal ($0.23^{2}$)
is equal to the desired signal-to-noise ratio (SNR). 

To confirm with \cite{MullerSmolaRatschScholkopfKohlmorgenVapnik1998},
the Mackey-Glass signal was down-sampled so that $nh=1$ and $n=1$.
The spline smoothing method would fare even better if we chose $h=.1$.
The delay and the embedding dimension used for delay coordinates were
$\tau=6$ and $D=6$, as in \cite{MullerSmolaRatschScholkopfKohlmorgenVapnik1998}.
The size of the training set was $N=1000$. For cross-validation,
the $\gamma/2D$ parameter was varied over $\left\{ 0.1,1.5,10.0,50.0,100.0\right\} $,
and the $\Lambda$ parameter was varied over $\left\{ 10^{-8.5},10^{-8},\ldots,10^{-0.5}\right\} $
for least squares with or without spline smoothing but over $\left\{ 10^{-10},10^{-6},10^{-2},10^{2}\right\} $
for the more expensive support vector regression. For support vector
regression, the $\epsilon$ was varied over $\left\{ 0.01,0.05,0.25\right\} $.
The phenomenon we will demonstrate is far more pronounced than the
slight gains obtained using more extensive cross-validation. For support
vector regression, we were able to reproduce the relevant results
reported in \cite{MullerSmolaRatschScholkopfKohlmorgenVapnik1998}.\footnote{The RMS error of $0.017$ reported for $t_{f}=1$ with SNR of 22.15\%
in \cite{MullerSmolaRatschScholkopfKohlmorgenVapnik1998} appears
to be a consequence of an unusually favorable noise or signal. The
typical RMS error is around $0.03$. We eliminate the effect of unusual
datasets by taking averages over multiple datasets.}

\begin{figure}
\centering{}\includegraphics[scale=0.6]{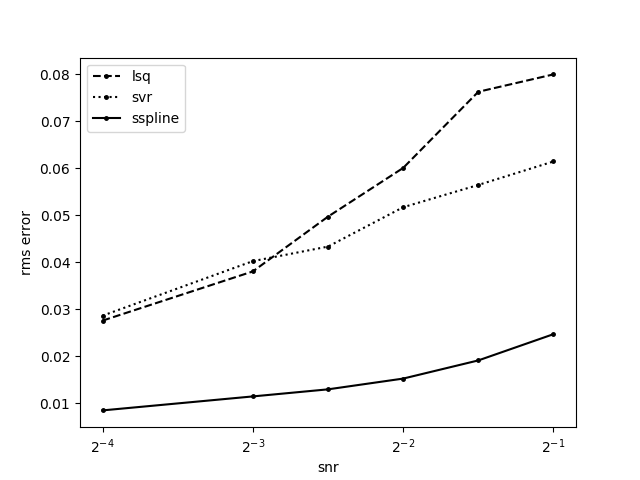}\caption{Root mean square errors in the prediction of the Mackey-Glass signal
with $t_{f}=1$ as a function of the signal to noise ratio. The superiority
of the method using smooth splines is evident.\label{fig:s3-mg-1step-rms-vs-snr}}
\end{figure}

Figure \ref{fig:s3-mg-1step-rms-vs-snr} demonstrates that (\ref{eq:algo-muller})
produces predictors that are corrupted by errors in the inputs or
delay coordinates. The method with spline smoothing is more accurate
and deteriorates less with increasing SNR. For the Mackey-Glass plots
in Figures \ref{fig:s3-mg-1step-rms-vs-snr}, \ref{fig:s2-itervsnoiter},
and \ref{fig:s3-mg-rms-vs-tf}, each point is an average over $480$
independent datasets in the case of least squares with or without
spline smoothing and over $48$ data sets in the case of support vector
regression. In all cases, using half as many datasets does not change
the picture.

\begin{figure}
\centering{}\includegraphics[scale=0.6]{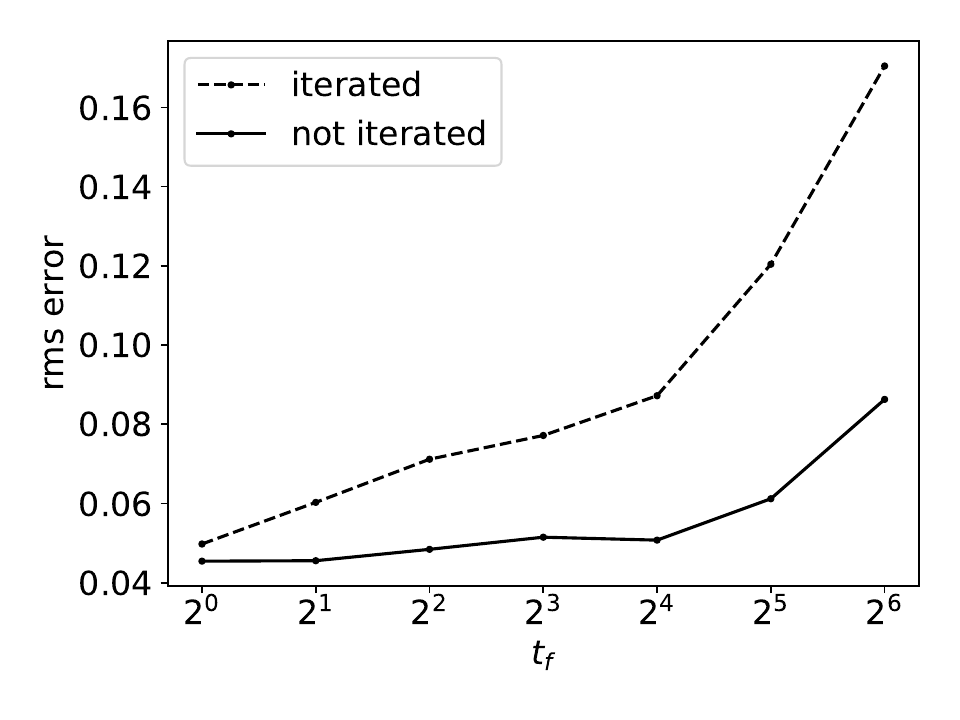}\caption{Comparison of the $1$-step least squares predictor (without spline
smoothing) iterated $t_{f}$ times with the $t_{f}$-step predictor
(without spline smoothing). The latter is seen to be superior.\label{fig:s2-itervsnoiter}}
\end{figure}
A $t_{f}=n_{f}\tau$ predictor can be obtained by iterating a $\tau$-step
predictor $n_{f}$ times, and this strategy is sometimes used to save
cost \cite{MullerSmolaRatschScholkopfKohlmorgenVapnik1998}. This
is not a good idea as explained in \cite{ViswanathLiangSerkh2013}
and as shown in Figure \ref{fig:s2-itervsnoiter}. An optimal predictor
would need to roughly split the distance to the nearest training sample
such that the component of the distance along unstable directions
is small and with the component along stable directions allowed to
be much larger. The balance between the two components depends upon
$t_{f}$, and therefore, iterating a one-step predictor is not a good
strategy. 

\begin{figure}
\centering{}\includegraphics[scale=0.45]{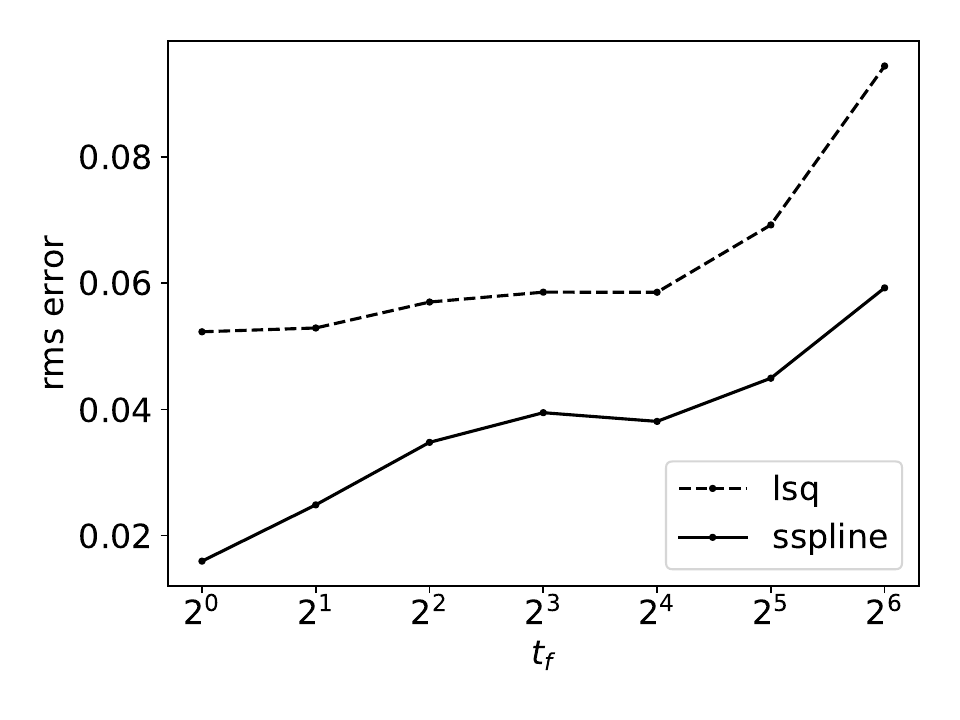}\hspace{0.25cm}\includegraphics[scale=0.45]{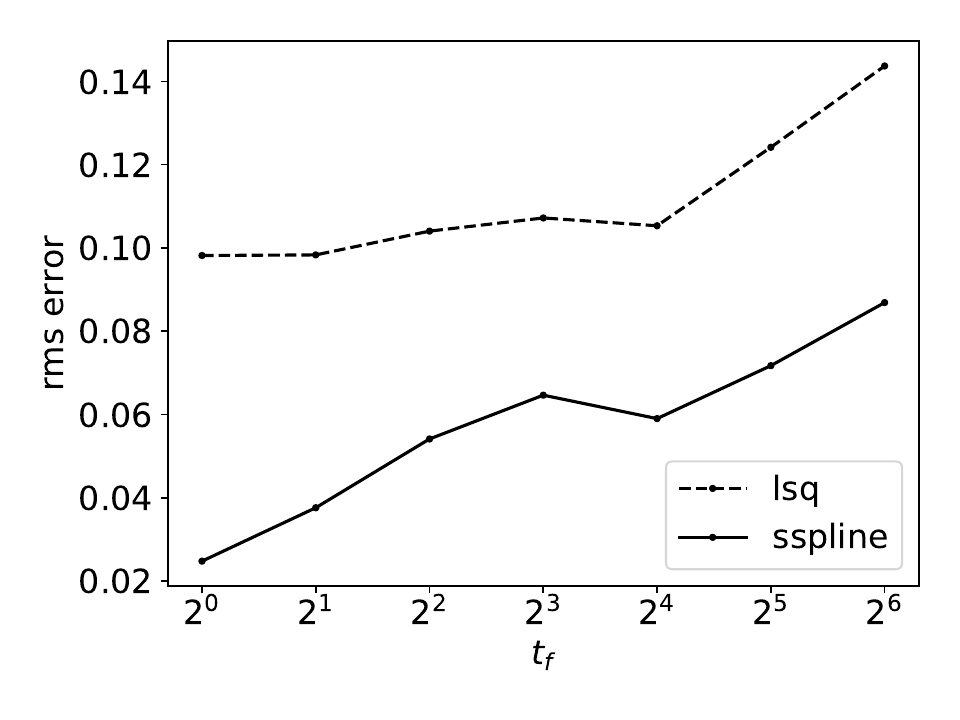}\caption{The plot on the left uses SNR of $0.2$ and the plot on the right
uses $0.4$. The method using smooth splines does better in all instances.\label{fig:s3-mg-rms-vs-tf}}
\end{figure}
 In Figure \ref{fig:s3-mg-1step-rms-vs-snr}, we see that spline smoothing
becomes more and more advantageous as noise increases. The situation
in Figure \ref{fig:s3-mg-rms-vs-tf} is a little different. When $t_{f}$
is small, spline smoothing does help more for the noisier SNR of $0.4$
compared to $0.2$. However, for larger $t_{f}$, even though spline
smoothing helps, it does not help more when the noise is higher. This
could be because as $t_{f}$ increases capturing the correct geometry
of the predictor becomes more and more difficult, and this difficulty
may be constraining the accuracy of the predictor.

The MacKey-Glass example is a delay-differential equation and does
not come under the purview of our convergence theorem. The Lorenz
example, $\dot{x}=10(y-x),\:\dot{y}=28x-y-xz,\:\dot{z}=-8z/3+xy$,
is a dynamical system with a compact invariant set and comes under
the purview of the convergence theorem. The Lorenz signal has a standard
deviation of $7.9$. For the Lorenz plots of Figure \ref{fig:s3-lrz-dt0.01-dt0.1},
each point is an average over $160$ datasets each with $N=1000$.
The picture did not change even with many fewer datasets. 

\begin{figure}
\centering{}\includegraphics[scale=0.35]{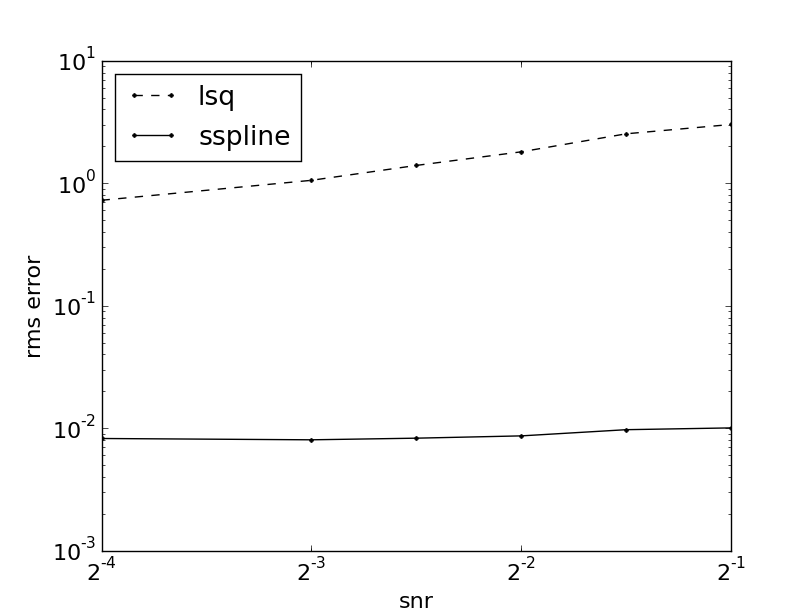}\hspace{0.25cm}\includegraphics[scale=0.35]{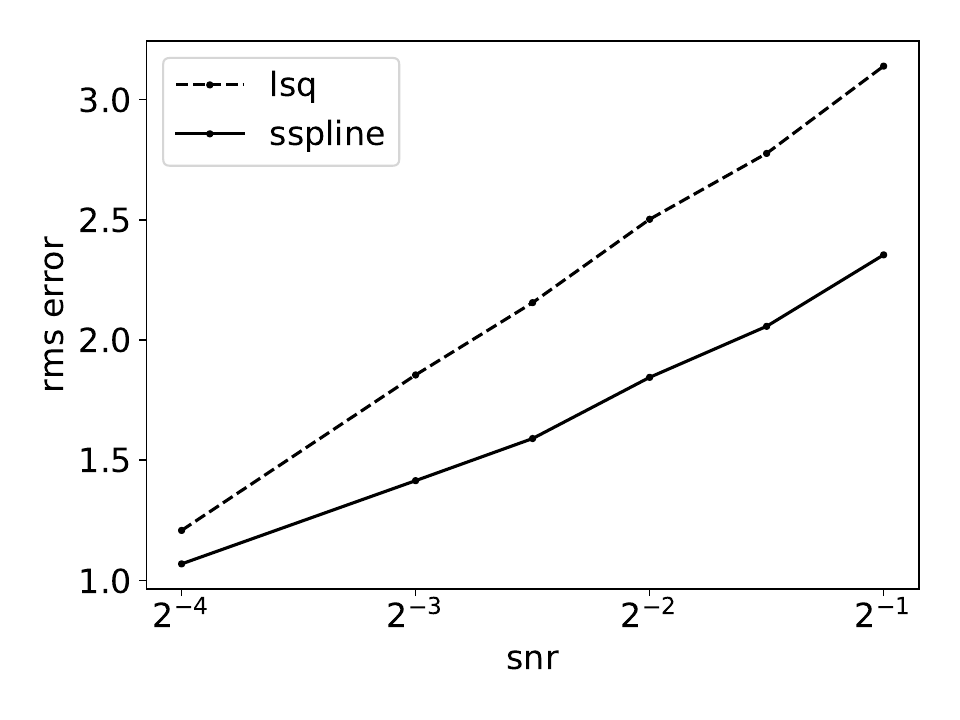}\caption{The advantage of spline smoothing for Lorenz is much less on the right
with $h=0.1$ than on the left with $h=0.01$. \label{fig:s3-lrz-dt0.01-dt0.1}}
\end{figure}
Figure \ref{fig:s3-lrz-dt0.01-dt0.1} compares $h=.01$ and $h=.1$
for Lorenz. In both cases, the embedding dimension is $d=10$, the
delay parameter is $\tau=1$, and the lookahead is $t_{f}=h$. It
may be seen that spline smoothing is less effective when $h=0.1$
as compared to $h=0.01$. A typical Lorenz oscillation has a period
of about $0.75$, and when $h=0.1$ the resolution is too low causing
too much discretization error. Smooth splines are less effective in
reconstructing the noise-free signal if the grid on the time axis
does not have sufficient resolution. The left half of Figure \ref{fig:s3-lrz-dt0.01-dt0.1}
shows an example where prediction using spline smoothing improves
accuracy by a factor of $100$ with $h=0.01$.

\section{Discussion}

For the prediction of dynamical time series, we have shown that flows
are quite different from maps. In the case of flows, the time series
can be denoised by relying solely on the smoothness of the underlying
flow. The predictor can be derived by applying kernel-based regression
to the denoised signal. The resulting predictor converges to the exact
predictor under conditions described by Theorem \ref{thm:convergence-theorem-main}.

As far as dynamical time series are concerned, the parameter estimation
problem \cite{McGoffMukherjeePillai2015,McGoffNobel2016} is complementary
to prediction. Much of the existing theory is for maps and with the
assumption of rapid mixing. For flows, smooth splines or a similar
technique may prove an effective method to denoise in the context
of parameter estimation as well.

The convergence theorem given here does not give rates and is not
uniform. Obtaining rates with uniformity over a class of flows will
probably require rapid mixing assumptions as in the case of maps \cite{HangFengSteinwartSuykens2016,SteinwartAnghel2009}.
Rapid mixing results for flows may be found in \cite{AraujoMelbourne2015}
for example.

With respect to rates and uniformity, there are two more issues that
would need to be considered. First, convergence of smooth splines
in the $\infty$-norm must be proved with explicit bounds that depend
only on the norm of the $m$-th derivative. A more significant point
is that rates of convergence for a given lookahead $t_{f}$ may not
be the best direction. As pointed out in \cite{ViswanathLiangSerkh2013},
the question of how large $t_{f}$ can be given a signal of length
$T$ appears to have implications for the prediction algorithm and
not just to its analysis. There is no evidence that existing algorithms
including the one in this paper are capable of predicting as far into
the future as an optimal algorithm should.

The smooth spline idea is primarily local and so are the optimality
results of Stone \cite{Stone1982}. Stone's algorithm for achievability
is to find a local scale and to fit a polynomial using linear least
squares within that local region. It is perhaps worth noting that
the same idea has a dynamical analog. In its dynamical version \cite{UrbanowiczHolyst2003},
the noisy dynamical time series is embedded within Euclidean space
using delay coordinates. The embedding will be necessarily noisy.
However, the embedded manifold can be smoothed locally using linear
techniques.

\section{Acknowledgements}

The authors thank the reviewer for a valuable report and the editors
for their comments.

\bibliographystyle{plain}
\bibliography{references}

\begin{thebibliography}{10}

\bibitem{AdamsNobel1998}
T.M. Adams and A.B. Nobel.
\newblock On density estimation from ergodic processes.
\newblock {\em The Annals of Probability}, 25:794--804, 1998.

\bibitem{AraujoMelbourne2015}
V.~Ara{\'u}jo and I.~Melbourne.
\newblock Exponential decay of correlations for nonuniformly hyperbolic flows
  with a $c^{1+\alpha}$ stable foliation, including the classical {Lorenz}
  attractor.
\newblock {\em Ann. Henri Poincar{\'e}}, 17:2975--3004, 2015.

\bibitem{ChristmannSteinwart2007}
A.~Christmann and I.~Steinwart.
\newblock Consistency and robustness of kernel-based regression in convex risk
  minimization.
\newblock {\em Bernoulli}, 3:799--819, 2007.

\bibitem{deBoor2001}
C.~de~Boor.
\newblock {\em A Practical Guide to Splines}.
\newblock Springer, New York, revised edition, 2001.

\bibitem{EggermontLariccia2006}
P.P.B. Eggermont and V.N. LaRiccia.
\newblock Uniform error bounds for smoothing splines.
\newblock In {\em High Dimensional Probability}, volume~51 of {\em IMS Lecture
  Notes-Monograph Series}, pages 220--237. Institute of Mathematical
  Statistics, 2006.

\bibitem{EggermontLaRiccia2009}
P.P.B. Eggermont and V.N. LaRiccia.
\newblock {\em Maximum Penalized Likelihood Estimation}, volume~II.
\newblock Springer, New York, 2009.
\newblock Springer Series in Statistics.

\bibitem{EkelandTemam1987}
I.~Ekeland and R.~Temam.
\newblock {\em Convex Analysis and Variational Problems}.
\newblock SIAM, Philadelphia, 1987.

\bibitem{FarmerSidorowich1987}
J.~Doyne Farmer and J.J. Sidorowich.
\newblock Predicting chaotic time series.
\newblock {\em Physical Review Letters}, 59:845--848, 1987.

\bibitem{GyorfiHardleSardaVieu2013}
L.~Gy{\"o}rfi, W.~H{\"a}rdle, P~Sarda, and P.~Vieu.
\newblock {\em Nonparametric Curve Estimation from Time Series}, volume~60 of
  {\em Lecture Notes in Statistics}.
\newblock Springer, New York, 1989.

\bibitem{GyorfiKohlerKrzyzakWalk2002}
L.~Gy{\"o}rfi, M.~Kohler, A.~Krzy{\.z}ak, and H.~Walk.
\newblock {\em A Distribution-Free Theory of Nonparametric Regression}.
\newblock Springer, New York, 2002.

\bibitem{HangFengSteinwartSuykens2016}
H.~Hang, Y.~Feng, I.~Steinwart, and J.A.K. Suykens.
\newblock Learning theory estimates with observations from general stationary
  stochastic processes.
\newblock {\em Neural computation}, 28:2853--2889, 2016.

\bibitem{Lalley1999}
S.~P. Lalley.
\newblock Beneath the noise, chaos.
\newblock {\em The Annals of Statistics}, 27:461--479, 1999.

\bibitem{LalleyNobel2006}
S.~P. Lalley and A.~B. Nobel.
\newblock Denoising deterministic time series.
\newblock {\em Dynamics of PDE}, 3:259--279, 2006.

\bibitem{Lalley2001}
S.P. Lalley.
\newblock Removing the noise from chaos plus noise.
\newblock In A.I. Mees, editor, {\em Nonlinear Dynamcs and Statistics}, pages
  233--244. Birkh{\"a}user, Boston, 2001.

\bibitem{MatteraHaykin1999}
D.~Mattera and S.~Haykin.
\newblock Support vector machines for dynamic reconstrution of a chaotic
  system.
\newblock In B.~Sch{\"o}lkopf, C.J.C. Burges, and A.J. Smola, editors, {\em
  Advances in Kernel Methods}, pages 211--242. MIT Press, MA, 1999.

\bibitem{McGoffMukherjeePillai2015}
K.~McGoff, S.~Mukherjee, and N.~Pillai.
\newblock Statistical inference for dynamical systems: A review.
\newblock {\em Statistics Surveys}, 9:209--252, 2015.

\bibitem{McGoffNobel2016}
K.~McGoff and A.B. Nobel.
\newblock Empirical risk minimization and complexity of dynamical models.
\newblock {\em www.arxiv.org}, 2016.

\bibitem{MukherjeeOsunaGirosi1997}
S.~Mukherjee, E.~Osuna, and F.~Girosi.
\newblock Nonlinear prediction of chaotic time series using a support vector
  machine.
\newblock In {\em Neural Networks for Signal Processing VII---Proceedings of
  the 1997 IEEE Workshop}, pages 511--520. IEEE, 1997.

\bibitem{MullerSmolaRatschScholkopfKohlmorgenVapnik1998}
K.-S. M{\"u}ller, A.J. Smola, G.~R{\"a}tsch, B.~Sch{\"o}lkopf, J.~Kohlmorgen,
  and V.N. Vapnik.
\newblock Using support vector machines for time series prediction.
\newblock In B.~Scholkopf, C.J.C. Burges, and S.~Mika, editors, {\em Advances
  in Kernel Methods}, pages 243--253. MIT Press, 1998.

\bibitem{SapankevychSankar2009}
N.~I. Sapankevych and R.~Sankar.
\newblock Time series prediction using support vector machines: a survey.
\newblock {\em IEEE Computational Intelligence Magazine}, pages 24--38, May
  2009.

\bibitem{SauerYorkeCasdagli1991}
T.~Sauer, J.~A. Yorke, and M.~Casdagli.
\newblock Embedology.
\newblock {\em Journal of Statistical Physics}, 65:579--616, 1991.

\bibitem{ScholkofSmola2002}
B.~Sch{\"o}lkopf and A.J. Smola.
\newblock {\em {Learning with Kernels}}.
\newblock MIT Press, MA, 2002.

\bibitem{Steinwart2001}
I.~Steinwart.
\newblock On the influence of the kernel on the consistency of support vector
  machines.
\newblock {\em Journal of Machine Learning Research}, 2:67--93, 2001.

\bibitem{SteinwartAnghel2009}
I.~Steinwart and M.~Anghel.
\newblock Consistency of support vector machines for forecasting the evolution
  of an unknown ergodic dynamical system from observations with unknown noise.
\newblock {\em The Annals of Statistics}, 37:841--875, 2009.

\bibitem{SteinwartHushScovel2009}
I.~Steinwart, D.~Hush, and C.~Scovel.
\newblock Learning from dependent observations.
\newblock {\em Journal of Multivariate Analysis}, 100:175--194, 2009.

\bibitem{Stone1982}
C.~Stone.
\newblock Optimal rates of convergence for nonparametric regression.
\newblock {\em Annals of Statistics}, 10:1040--1063, 1982.

\bibitem{UrbanowiczHolyst2003}
K.~Urbanowicz and J.A. Holyst.
\newblock Noise-level estimation of time series using coarse-grained entropy.
\newblock {\em Physical Review E}, 67:046218, 2003.

\bibitem{ViswanathLiangSerkh2013}
D.~Viswanath, X.~Liang, and K.~Serkh.
\newblock Metric entropy and the optimal prediction of chaotic signals.
\newblock {\em SIAM Journal on Applied Dynamical Systems}, 12:1085--1113, 2013.

\bibitem{Wahba1990}
G.~Wahba.
\newblock {\em Spline Models for Observational Data}.
\newblock SIAM, Philadelphia, 1990.

\bibitem{Zhou2008}
D.-X. Zhou.
\newblock Derivative reproducing properties for kernel methods in learning
  theory.
\newblock {\em Journal of Computational and Applied Mathematics}, 220:456--463,
  2008.

\end{thebibliography}

\end{document}